\documentclass[preprint,12pt,authoryear]{elsarticle}

\usepackage{booktabs}
\usepackage{amsmath,amssymb,amsthm}
\usepackage{graphicx,subcaption}
\usepackage{mathdots}
\usepackage{multirow}
\usepackage{csquotes}
\usepackage[capitalize,noabbrev]{cleveref}
\usepackage{enumitem}
\graphicspath{ {./figure/} }

\numberwithin{equation}{section}

\newtheorem{theorem}{Theorem}[section]
\newtheorem{lemma}[theorem]{Lemma}

\theoremstyle{remark}
\newtheorem{remark}[theorem]{Remark}
\newtheorem{assumption}[theorem]{Assumption}
\crefname{assumption}{Assumption}{Assumptions}

\newcommand\R{\mathbb{R}}
\newcommand\bO{\mathbf{O}}

\newcommand\bW{\mathbf{W}}
\newcommand\bx{\mathbf{x}}

\newcommand\by{\mathbf{y}}

\newcommand\E{\mathcal{E}}
\newcommand\rT{\mathrm{T}}
\newcommand\bbE{\mathbb{E}}

\newcommand\gp{\mathrm{gp}}
\newcommand\bal{\mathrm{bal}}
\newcommand\param{\mathrm{par}}

\usepackage[normalem]{ulem}
\usepackage{color}

\journal{Neural Networks}

\begin{document}

\begin{frontmatter}



\title{Balanced Group Convolution: An Improved Group Convolution Based on Approximability Estimates}


\author[a]{Youngkyu Lee}
\author[b]{Jongho Park}
\author[c]{Chang-Ock Lee}

\affiliation[a]{organization={Natural Science Research Institute, KAIST},
            state={Daejeon},
            postcode={34141}, 
            country={Korea}}
\affiliation[b]{organization={Computer, Electrical and Mathematical Sciences and Engineering Division, King Abdullah University of Science and Technology~(KAUST)},
state={Thuwal},
postcode={23955}, 
country={Saudi Arabia}}
\affiliation[c]{organization={Department of Mathematical Sciences, KAIST},
state={Daejeon},
postcode={34141}, 
country={Korea}}

\begin{abstract}
The performance of neural networks has been significantly improved by increasing the number of channels in convolutional layers.
However, this increase in performance comes with a higher computational cost, resulting in numerous studies focused on reducing it.
One promising approach to address this issue is group convolution, which effectively reduces the computational cost by grouping channels.
However, to the best of our knowledge, there has been no theoretical analysis on how well the group convolution approximates the standard convolution.
In this paper, we mathematically analyze the approximation of the group convolution to the standard convolution with respect to the number of groups.
Furthermore, we propose a novel variant of the group convolution called \textit{balanced group convolution}, which shows a higher approximation with a small additional computational cost.
We provide experimental results that validate our theoretical findings and demonstrate the superior performance of the balanced group convolution over other variants of group convolution.
\end{abstract}



\begin{keyword}


Convolutional layer \sep Group convolution \sep Approximability estimate
\end{keyword}

\end{frontmatter}


\section{Introduction}
\label{Sec:Int}
The convolutional layer plays a crucial role in the success of modern neural networks for image classification~\citep{he2016deep,hu2018squeeze,krizhevsky2012imagenet} and image processing problems~\citep{radford2015unsupervised,ronneberger2015u,zhang2017beyond}.
However, achieving high performance through convolutional neural networks~(CNNs) typically requires the use of a large number of channels~\citep{brock2021high,he2022approximation,huang2019gpipe,tan2019efficientnet,zagoruyko2016wide}, resulting in significant computational costs and long training times.
Accordingly, there have been many studies focusing on modifying the convolutional layer to reduce its computational complexity~\citep{gholami2021survey,krizhevsky2012imagenet,lee2022two,liu2020pruning,zhang2015accelerating,zhang2015efficient}.

Among them, group convolution is the most basic modification that can be easily thought of.
It was introduced in AlexNet~\citep{krizhevsky2012imagenet} as a distributed computing method of convolutional layers to resolve the memory shortage.
Group convolution divides the channels of each layer into groups and performs convolution only within each group.
CNNs with the group convolution succeeded in reducing the number of parameters and implementation time, but its performance drops rapidly as the number of groups increases~\citep{lee2022two,long2015fully}.
It has been speculated that this phenomenon is due to the fact that the lack of intergroup communication greatly affects the representation capacity of CNNs.
However, it has not yet been mathematically revealed how much the performance is degraded.

Recently, several modifications of the group convolution, which add an intergroup communication, have been considered to restore the performance \citep{chollet2017xception,lee2022two,long2015fully,zhang2018shufflenet}.
The channel shuffling structure used in ShuffleNet~\citep{zhang2018shufflenet} adds a permutation step among the output channels of the group convolution to make data exchange among groups.
This method is efficient in terms of memory usage because it does not use any additional parameters.
Learnable group convolution~\citep{Huang_2018_CVPR} determines channels to be grouped through learning.
This method generates the weight for the group convolution by overlaying a trainable mask on the weight of the standard convolution.
In fact, it is equivalent to changing the channel shuffling from a fixed rule to a learnable one.
Although these methods do not make sense when viewed as a single layer, they effectively restore the performance of CNNs that use multiple layers.
On the other hand, fully learnable group convolution~\citep{zhang2018shufflenet} introduces more parameters to additionally vary the number of channels in each group.
In~\citep{lee2022two}, it was observed that group convolution has a block diagonal matrix structure and two-level group convolution was introduced to collect representatives of each group to perform additional convolution with low computational cost.
These works have successfully resolved the performance degradation issue described above, but there is still no mathematical analysis of why the performance is improved.

In this paper, we rigorously analyze the performance of the group convolution as an approximation to the corresponding standard convolution with respect to the number of groups $N$.
To achieve this, we introduce a new measure of approximability~(see~\eqref{eqn:app_err}), which quantifies the optimal squared $\ell^2$-error between the outputs of the standard convolution and group convolution.
We prove that the approximability of the group convolution is estimated as $K(1 - 1/N)^{2}$, where $K$ denotes the number of parameters in a convolution layer that maps a single channel to a single channel~(see \cref{thm:1}).

In addition, we present a new variant of group convolution, called \textit{balanced group convolution}~(BGC), which achieves a theoretically improved approximability estimate compared to the plain group convolution.
In BGC, the intergroup mean is computed by averaging the channels between groups.
After that, we introduce an additional convolution to add the intergroup mean to the output of the group convolution.
This allows each group in BGC to utilize the entire information of the input, resulting in an improved approximability with a small additional computational cost.
Indeed, we prove that the approximability of BGC is estimated to be $K(1-1/N)^{3}$~(see \cref{thm:2}), which is a better result than the plain group convolution.
Furthermore, under an additional assumption, the approximability of the group convolution and BGC can be estimated to be $K(1-1/N)/n$ and $K(1-1/N)^{2}/n$, respectively, where $n$ is the number of input channels.
The superior performance of the proposed BGC is verified by various numerical experiments on several recent CNNs such as WideResNet~\citep{zagoruyko2016wide}, ResNeXt~\citep{xie2017aggregated}, MobileNetV2~\citep{sandler2018mobilenetv2}, and EfficientNet~\citep{tan2019efficientnet}.

In summary, we address the following issues in this work:
\begin{displayquote}
    \textit{Is it possible to obtain a rigorous estimate for the approximability of group convolution? Furthermore, can we develop an enhanced version of group convolution that achieves a theoretically better approximability than the original?}
\end{displayquote}

We summarize our main contributions in the followings:
\begin{itemize}
    \item We propose BGC, which is an enhanced variant of group convolution with an improved approximability estimate.
    \item We estimate the bounds on approximability of group convolution and BGC.
    \item We demonstrate the performance of BGC by embedding it into state-of-the-art CNNs.
\end{itemize}

The rest of this paper is organized as follows.
In \cref{Sec:GC}, we present preliminaries of this paper, specifically related to the group convolution.
We introduce the proposed BGC and its improved approximation properties in \cref{Sec:Main}, and establish connections to existing group convolution approaches in \cref{Sec:Comparison}.
Numerical validations of the theoretical results and improved classification accuracy of the proposed BGC applied to various state-of-the-art CNNs are presented in \cref{Sec:Num}.
We conclude this paper with remarks in \cref{Sec:Conc}.

\section{Group convolution}
\label{Sec:GC}
In this section, we introduce some basic notations for group convolution to be used throughout this paper.
Let $m$ and $n$ be positive integers and $N$ be a common divisor of $m$ and $n$.
We write $m_N = m/N$ and $n_N = n/N$.

We consider convolutional layers that operate on $n$ channels and produces $m$ channels.
Input $\bx$ of a convolutional layer is written as $\bx = (x_{1}, x_{2}, \dots, x_{n})$, where $x_{j}$ is the $j$th channel of $\bx$.
Similarly, output $\by$ is written as $\by = (y_{1}, y_{2}, \dots, y_{m})$, where $y_{i}$ is the $i$th channel of $\by$.
To represent a standard convolutional layer that takes input $\bx$ and produces output $\by$, we use the following matrix expression:
\begin{equation}
\label{eqn:full}
\begin{bmatrix} y_{1} \\ y_{2} \\ \vdots \\ y_{m} \end{bmatrix}
= \begin{bmatrix} W_{11} & W_{12} & \cdots & W_{1n} \\
    W_{21} & W_{22} & \cdots & W_{2n} \\
    \vdots & \vdots & \ddots & \vdots \\
    W_{m1} & W_{m2} & \cdots & W_{mn} \end{bmatrix}
\begin{bmatrix} x_{1} \\ x_{2} \\ \vdots \\ x_{n} \end{bmatrix},
\end{equation}
where $W_{ij}$ is a matrix that represents the generic convolution that maps $x_{j}$ to $y_{i}$.
We suppose that the channels of $\bx$ and $\by$ are evenly partitioned into $N$ groups.
Namely, let $\bx = (\bx^1, \bx^2, \dots, \bx^N)$ and $\by = (\by^1, \dots, \by^N)$, where
\begin{equation*}
    \bx^k = \begin{bmatrix} x_{(k-1)n_N + 1} \\ x_{(k-1)n_N + 2} \\ \vdots \\ x_{k n_N} \end{bmatrix}, \quad
    \by^k = \begin{bmatrix} y_{(k-1)m_N + 1} \\ y_{(k-1)m_N + 2} \\ \vdots \\ y_{k m_N} \end{bmatrix}.
\end{equation*}
Then~\eqref{eqn:full} is rewritten as the following block form:
\begin{equation}
\label{eqn:full_group}
    \begin{bmatrix} \by^1 \\ \by^2 \\ \vdots \\ \by^N \end{bmatrix}
    = \begin{bmatrix} \bW^{11} & \bW^{12} & \cdots & \bW^{1N} \\
        \bW^{21} & \bW^{22} & \cdots & \bW^{2N} \\
        \vdots & \vdots & \ddots & \vdots \\
        \bW^{N1} & \bW^{N2} & \cdots & \bW^{NN} \end{bmatrix}
    \begin{bmatrix} \bx^1 \\ \bx^2 \\ \vdots \\ \bx^N \end{bmatrix}.
\end{equation}

Group convolution, a computationally efficient version of the convolution proposed in~\citep{krizhevsky2012imagenet}, is formed by discarding intergroup connection in~\eqref{eqn:full_group}.
That is, we take only the block-diagonal part to obtain the group convolution, such as
\begin{equation}
\label{eqn:group}
    \begin{bmatrix} \by^1 \\ \by^2 \\ \vdots \\ \by^N \end{bmatrix}
    = \begin{bmatrix} \bW^{11} & \bO & \cdots & \bO \\
        \bO & \bW^{22} & \cdots & \bO \\
        \vdots & \vdots & \ddots & \vdots \\
        \bO & \bO & \cdots & \bW^{NN} \end{bmatrix}
    \begin{bmatrix} \bx^1 \\ \bx^2 \\ \vdots \\ \bx^N \end{bmatrix},
\end{equation}
where $\bO$ denotes the zero matrix.
We observe that the number of parameters in the group convolution~\eqref{eqn:group} is $1/N$ of that of the standard convolution~\eqref{eqn:full_group}.
Moreover, the representation matrix in~\eqref{eqn:group} is block-diagonal, allowing each block on the main diagonal to be processed in parallel, making it computationally more efficient than~\eqref{eqn:full_group}. 
However, when the number of parameters in the group convolution is reduced, the performance of a CNN using group convolutional layers may decrease compared to a CNN using standard convolutional layers~\citep{lee2022two,zhang2018shufflenet}.

\section{Main results}
\label{Sec:Main}
In this section, we present the main results of this paper, including the proposed BGC and estimates of its approximability and computational cost.

\subsection{Balanced group convolution}
A major drawback of the group convolution, which affects performance, is the lack of intergroup communication since the representation matrix given in~\eqref{eqn:group} is block-diagonal.
The motivation behind the proposed BGC is by adding intergroup communication to the group convolution with low computational cost to improve performance.
To achieve this, we utilize a representative vector of small dimension that captures all the information in the input $\bx$. 
Specifically, we use the mean $\bar{\bx}$ of $\{ \bx^k \}$, i.e.,
\begin{equation*}
    \bar{\bx} = \frac{1}{N}\sum_{k=1}^N \bx^k.
\end{equation*}
Then, the proposed BGC is defined by
\begin{equation}
    \label{eqn:balanced}
    \begin{bmatrix} \by^1 \\ \by^2 \\ \vdots \\ \by^N \end{bmatrix}
    = \begin{bmatrix} \bW^{11} & \bO & \cdots & \bO \\
        \bO & \bW^{22} & \cdots & \bO \\
        \vdots & \vdots & \ddots & \vdots \\
        \bO & \bO & \cdots & \bW^{NN} \end{bmatrix}
    \begin{bmatrix} \bx^1 \\ \bx^2 \\ \vdots \\ \bx^N \end{bmatrix}
    + \begin{bmatrix} \bar{\bW}^1 \\ \bar{\bW}^2 \\ \vdots \\ \bar{\bW}^N \end{bmatrix}
    \bar{\bx},
\end{equation}
where each $\bar{\bW}^k$ is a convolution that operates on $n_N$ channels and produces $m_N$ channels.
That is, BGC is an extension of the group convolution that incorporates an additional structure based on the intergroup mean $\bar{\bx}$, which serves as a balancing factor that utilizes the entire information of $\bx$ to intergroup communication.
Specifically, this additional structure allows BGC to better distribute feature representations across groups.

Since the number of parameters in BGC~\eqref{eqn:balanced} is $2/N$ of that of the standard convolution~\eqref{eqn:full_group}, the computational cost of BGC is significantly lower than that of the standard convolution. We will discuss this further in \cref{Sec:Main}.3.

\subsection{Approximability estimate}
Thanks to the additional structure in the proposed BGC, we may expect that it behaves more similarly to the standard convolution than the group convolution.
To provide a rigorous assessment of this behavior, we introduce a notion of approximability and prove that BGC has a better approximability to the standard convolution than the group convolution.

Suppose that the standard convolution $\bW = [\bW^{kl}]$ that maps $N$ groups to $N$ groups~(see~\eqref{eqn:full_group}) and the input $\bx = [\bx^k]$ are drawn from a certain probability distribution.
We define the approximability of the group convolution $\bW_{\gp}$ of the form~\eqref{eqn:group} as
\begin{equation}
    \label{eqn:app_err}
    \bbE_{\bW} \left[ \inf_{\bW_{\gp}} \bbE_{\bx} \left[ \| \bW \bx - \bW_{\gp} \bx \|^2 \right] \right],
\end{equation}
where $\| \cdot \|$ denotes the $\ell^2$-norm.
Namely, the approximability~\eqref{eqn:app_err} measures the $\bW$-expectation of the optimal $\bx$-averaged squared $\ell^2$-error between the output $\bW \bx$ of the standard convolution and the output $\bW_{\gp} \bx$ of the group convolution.
The approximability of BGC is defined in the same manner.

To estimate the approximability~\eqref{eqn:app_err}, we need the following assumptions on the distributions of $\bW$ and $\bx$.

\begin{assumption}
\label{ass:data}
A standard convolution $\bW = [W_{ij}]$ and an input $\bx = [x_{j}]$ are random variables that satisfy the following:
\begin{enumerate}[label=(\roman*)]
    \item $\{ W_{ij} \}$ are identically distributed.
    \item $\{ x_{j} \}$ are independent and identically distributed~(i.i.d.).
    \item $\bW$ and $\bx$ are independent.
\end{enumerate}
\end{assumption}
\cref{ass:data}(i) is essential to handle the off-diagonal parts of the group convolution when estimating~\eqref{eqn:app_err}.
To effectively compensate for the absence of the off-diagonal parts in the group convolution using $\bar{\bx}$, we need \cref{ass:data}(ii).
On the other hand, \cref{ass:data}(iii) is a quite natural assumption since it asserts that the convolution and input are independent of each other.
Furthermore, if we additionally assume to the parameters of the standard convolution, we can get better estimates on the approximability of group convolution and BGC.
This assumption is specified in~\cref{ass:input}.
\begin{assumption}
\label{ass:input}
A standard convolution $\bW = [W_{ij}]$ is a random variable such that $\{ W_{ij} \}$ are i.i.d. with $\bbE [ W_{ij} ] = 0$.
\end{assumption}
The independence of $W_{ij}$ is essential to obtain a sharper bound on the expectation of $\| \bW\bx \|^{2}$.
The condition $\bbE [ W_{ij} ] = 0$ usually occurs when the parameters of the convolutional layers are generated by Glorot~\citep{glorot10a} or He~\citep{he2015delving} initialization, which are i.i.d. samples from random variable with zero mean.

The following lemma presents a Young-type inequality for the standard convolution, which plays a key role in the proof of the main theorems.

\begin{lemma}
\label{lem:Young}
Let $\bW$ be a standard convolution that operates on $n$ channels and produces $m$ channels.
For any $n$-channel input $\bx$, we have
\begin{equation*}
\| \bW \bx \| \leq K^{\frac{1}{2}} \| \bW \|_{\param} \| \bx \|,
\end{equation*}
where $K$ is the number of parameters in a convolutional layer that maps a single channel to a single channel and $\| \bW \|_{\param}$ denotes the $\ell^2$-norm of the vector of parameters of the convolution $\bW$.
\end{lemma}
\begin{proof}
    We first prove the case $m = n =1$.
    Let $W$ be a convolution that maps a single channel to a single channel and $x \in \mathbb{R}^D$ be an input.
    Since the output $Wx$ of the convolution is linear with respect to the parameters of $W$, there exists a matrix $X \in \mathbb{R}^{D \times K}$ such that
    \begin{equation*}
        W x = X w,
    \end{equation*}
    where $w \in \mathbb{R}^K$ is the vector of parameters of $W$.
    Note that each entry of $X$ is also an entry of $x$.
    For each entry $x_d$~($1 \leq d \leq D$) of $x$ and each entry $w_k$~($1 \leq k \leq K$) of $w$, their product $w_k x_d$ appears at most one in the formulation of $W x$.
    Hence, the $k$th column $v_k$ of $X$ satisfies
    \begin{equation}
        \label{lem1:Young}
        \| v_k \| \leq \| x \|.
    \end{equation}
    By the triangle inequality, Cauchy--Schwarz inequality, and~\eqref{lem1:Young}, we obtain
    \begin{equation}
        \label{lem2:Young}
        \begin{split}
        \| W x \| &= \| X w \| \leq \sum_{k=1}^K \| v_k w_k \| \\
        &\leq \left( \sum_{k=1}^K \| v_k \|^2 \right)^{\frac{1}{2}} \left( \sum_{k=1}^K w_k^2 \right)^{\frac{1}{2}}
        \leq K^{\frac{1}{2}} \| W \|_{\param} \| x \|,
        \end{split}
    \end{equation}
    which completes the proof of the case $m = n = 1$.

    Next, we consider the general case.
    As in~\eqref{eqn:full}, we write $\bW = [W_{ij}]$ and $\bx = [x_{j}]$.
    By the triangle inequality,~\eqref{lem2:Young}, and the Cauchy--Schwarz inequality, we obtain
    \begin{align*}
        \| \bW \bx \|^2 &= \sum_{i=1}^m \left\| \sum_{j=1}^n W_{ij} x_{j} \right\|^2
        \leq \sum_{i=1}^m \left( \sum_{j=1}^n \| W_{ij} x_{j} \| \right)^2 \\
        &\leq K \sum_{i=1}^m \left( \sum_{j=1}^n \| W_{ij} \|_{\param} \|x_{j} \| \right)^2 \\
        &\leq K \sum_{i=1}^m \left( \sum_{j=1}^n \| W_{ij} \|_{\param}^2 \cdot \sum_{j=1}^n \| x_{j} \|^2 \right) \\
        &=  K \| \bW \|_{\param}^2 \| \bx \|^2,
    \end{align*}
    which is our desired result.
\end{proof}

A direct consequence of \cref{lem:Young} is that, if $\bW$ and $\bx$ are independent, then we have
\begin{equation}
\label{eqn:Young}
    \bbE_{\bW} \bbE_{\bx} \left[ \| \bW \bx \|^2 \right] \leq K \bbE \left[ \| \bW \|_{\param}^2 \right] \bbE \left[ \| \bx \|^2 \right].
\end{equation}
In the following lemma, we show that the above estimate can be improved up to a multiplicative factor $1/n$ if we additionally assume that $W_{ij}$ has zero mean and that some random variables are independent and/or identically distributed.

\begin{lemma}
    \label{lem:Young_improved}
    Let $\bW = [W_{ij}]$ be a standard convolution that operates on $n$ channels and produces $m$ channels, and let $\bx = [x_j]$ be an $n$-channel input.
    Assume that $\bW$ and $\bx$ satisfy the following:
    \begin{enumerate}[label=\emph{(\roman*)}]
        \item $\{ W_{ij} \}$ are i.i.d. with $\bbE [W_{ij}] = 0$.
        \item $\{ x_j \}$ are identically distributed.
        \item $\bW$ and $\bx$ are independent.
    \end{enumerate}
    Then we have
    \begin{equation*}
        \bbE_{\bW} \bbE_{\bx} \left[ \| \bW \bx \|^2 \right]
        \leq \frac{K}{n} \bbE \left[ \| \bW \|_{\param}^2 \right] \bbE \left[ \| \bx \|^2 \right].
    \end{equation*}
\end{lemma}
\begin{proof}
We first prove the case $m = 1$.
For $i \neq j$, invoking~(i) yields
\begin{equation}
    \label{lem1:Young_improved}
    \bbE_{\bW} \bbE_{\bx} \left[ x_i^{\rT} W_{1i}^{\rT} W_{1j} x_j \right]
    = \bbE_{\bx} \left[ x_i^\rT  \bbE \left[ W_{1i}^{\rT}\right] \bbE \left[ W_{1j} \right] x_j \right] = 0.
\end{equation}
On the other hand, for $1 \leq j \leq n$, by~(iii) and~\eqref{eqn:Young}, we have
\begin{equation}
    \label{lem2:Young_improved}
    \bbE_{W_{1j}} \bbE_{x_j} \left[ \| W_{1j} x_j \|^2 \right]
    \leq K \bbE \left[ \| W_{1j} \|_{\param}^2 \right] \bbE \left[ \|x_j\|^2 \right]
    = \frac{K}{n^2} \bbE \left[ \| \bW \|_{\param}^2 \right] \bbE \left[ \| \bx \|^2 \right],
\end{equation}
where the last equality is due to (i) and (ii).
By~\eqref{lem1:Young_improved} and~\eqref{lem2:Young_improved}, we get
\begin{align*}
    \bbE_{\bW} \bbE_{\bx} \left[ \| \bW \bx \|^2 \right]
    &= \sum_{i=1}^n \sum_{j=1}^n \bbE_{\bW} \bbE_{\bx} \left[ x_i^{\rT} W_{1i}^{\rT} W_{1j} x_j \right] \\
    &\leq \sum_{j=1}^n \frac{K}{n^2} \bbE \left[ \| \bW \|_{\param}^2 \right] \bbE \left[ \| \bx \|^2 \right]
    = \frac{K}{n} \bbE \left[ \| \bW \|_{\param}^2 \right] \bbE \left[ \| \bx \|^2 \right],
\end{align*}
which completes the proof of the case $m = 1$.
The general case $m > 1$ can be shown by applying the case $m = 1$ to each $i$th output channel~($1 \leq i \leq m$) and then summing over all $i$.
\end{proof}

The representation matrix of the group convolution presented in~\eqref{eqn:group} becomes more sparse as the number of groups $N$ increases.
This suggests that the approximability of the group convolution decreases as $N$ increases.
However, to the best of our knowledge, there has been no theoretical analysis on this.
\cref{thm:1} presents the first theoretical result regarding the approximability of the group convolution.

\begin{theorem}
  \label{thm:1}
  For an $n$-channel input $\bx$, let $\bW \bx$ and $\bW_{\gp} \bx$ denote the outputs of the standard and group convolutions given in~\eqref{eqn:full_group} and~\eqref{eqn:group}, respectively.
  Under \cref{ass:data}, we have
  \begin{equation*}
  \bbE_{\bW} \left[ \inf_{\bW_{\gp}} \bbE_{\bx} \left[ \| \bW \bx - \bW_{\gp} \bx \|^2 \right] \right]
  \leq K \left( 1 - \frac{1}{N} \right)^2 \bbE \left[ \| \bW \|_{\param}^2 \right] \bbE \left[ \| \bx \|^2 \right].
  \end{equation*}
  In addition, if we further assume that \cref{ass:input} holds, then we have
  \begin{equation*}
  \bbE_{\bW} \left[ \inf_{\bW_{\gp}} \bbE_{\bx} \left[ \| \bW \bx - \bW_{\bal} \bx \|^2 \right] \right]
  \leq \frac{K}{n} \left( 1 - \frac{1}{N} \right) \bbE \left[ \| \bW \|_{\param}^2 \right] \bbE \left[ \| \bx \|^2 \right].
  \end{equation*}
\end{theorem}

\begin{proof}
Take any standard convolution $\bW = [\bW^{kl}]$ and any input $\bx = [\bx^k]$ as in~\eqref{eqn:full_group}.
If we set
\begin{equation*}
    \bW_{\gp} = \begin{bmatrix} \bW^{11} & \bO & \cdots & \bO \\
        \bO & \bW^{22} & \cdots & \bO \\
        \vdots & \vdots & \ddots & \vdots \\
        \bO & \bO & \cdots & \bW^{NN} \end{bmatrix},
\end{equation*}
i.e., if we choose $\bW_{\gp}$ as the block-diagonal part of $\bW$, then we have
\begin{align}
    \label{thm1:1}
    \bbE_{\bW} \left[ \inf_{\bW_{\gp}} \bbE_{\bx} \left[ \| \bW \bx - \bW_{\gp} \bx \|^2 \right] \right]
    &\leq \bbE_{\bW} \bbE_{\bx} \left[ \| \bW \bx - \bW_{\gp} \bx \|^2 \right] \nonumber \\
    &= \sum_{k=1}^N \bbE_{\bW} \bbE_{\bx} \left[ \left\| \sum_{l \neq k} \bW^{kl} \bx^l \right\|^2 \right].
\end{align}
For each $k$, since the map $\bx \mapsto \sum_{l \neq k} \bW^{kl} \bx^l$ is a convolution that operates on $(n - n/N)$ channels~($\bx_k$ is excluded from the input), invoking \cref{lem:Young,lem:Young_improved} yields
\begin{equation}
    \label{thm2:1}
    \bbE_{\bW} \bbE_{\bx} \left[ \left\| \sum_{l \neq k} \bW^{kl} \bx^l \right\|^2 \right] \leq C_{n,N,k} \left( \sum_{l \neq k} \bbE \left[ \| \bW^{kl} \|_{\param}^2 \right] \right) \left( \sum_{l \neq k} \bbE \left[ \| \bx^l \|^2 \right] \right),
\end{equation}
where
\begin{equation}
\label{C}
C_{n,N,k} = \begin{cases}
K, & \quad \text{ under \cref{ass:data},} \\
\frac{K}{n-n/N}, & \quad \text{ under \cref{ass:data,ass:input}.}
\end{cases}
\end{equation}
Note that the independence condition between $\bW$ and $\bx$ are used in~\eqref{thm2:1}.

Meanwhile, \cref{ass:data} implies that
\begin{equation}
    \label{thm3:1}
    \bbE \left[ \| \bW^{kl} \|_{\param}^2 \right] = \frac{1}{N^2} \bbE \left[ \| \bW \|_{\param}^2 \right], \quad
    \bbE \left[ \| \bx^l \|^2 \right] = \frac{1}{N} \bbE \left[ \| \bx \|^2 \right].
\end{equation}
Finally, combination of~\eqref{thm1:1},~\eqref{thm2:1}, and~\eqref{thm3:1} yields
\begin{equation*}
    \bbE_{\bW} \left[ \inf_{\bW_{\gp}} \bbE_{\bx} \left[ \| \bW \bx - \bW_{\gp} \bx \|^2 \right] \right]
    \leq C_{n,N,k} \left( 1 - \frac{1}{N} \right)^2 \bbE \left[ \| \bW \|_{\param}^2 \right] \bbE \left[ \| \bx \|^2 \right],
\end{equation*}
which completes the proof.
\end{proof}

As discussed above, BGC~\eqref{eqn:balanced} has the additional structure that utilizes the intergroup mean $\bar{\bx}$ to compensate for the absence of the off-diagonal parts in the group convolution.
We obtain the main theoretical result summarized in \cref{thm:2}, showing that BGC achieves an improved approximability estimate compared to the group convolution.

\begin{theorem}
  \label{thm:2}
  For an $n$-channel input $\bx$, let $\bW \bx$ and $\bW_{\bal} \bx$ denote the outputs of the standard convolution and BGC given in~\eqref{eqn:full_group} and~\eqref{eqn:balanced}, respectively.
  Under \cref{ass:data}, we have
  \begin{equation*}
  \bbE_{\bW} \left[ \inf_{\bW_{\bal}} \bbE_{\bx} \left[ \| \bW \bx - \bW_{\bal} \bx \|^2 \right] \right]
  \leq K \left( 1 - \frac{1}{N} \right)^3 \bbE \left[ \| \bW \|_{\param}^2 \right] \bbE \left[ \| \bx \|^2 \right].
  \end{equation*}
  In addition, if we further assume that \cref{ass:input} holds, we have
  \begin{equation*}
  \bbE_{\bW} \left[ \inf_{\bW_{\bal}} \bbE_{\bx} \left[ \| \bW \bx - \bW_{\bal} \bx \|^2 \right] \right]
  \leq \frac{K}{n} \left( 1 - \frac{1}{N} \right)^{2} \bbE \left[ \| \bW \|_{\param}^2 \right] \bbE \left[ \| \bx \|^2 \right].
  \end{equation*}
\end{theorem}

\begin{proof}
Take any standard convolution $\bW = [\bW^{kl}]$ and any input $\bx = [\bx^k]$ as in~\eqref{eqn:full_group}.
If we choose $\bW_{\bal}$ as
\begin{equation*}
    \bW_{\bal} \bx
    =
    \begin{bmatrix} \bW^{11} & \bO & \cdots & \bO \\
        \bO & \bW^{22} & \cdots & \bO \\
        \vdots & \vdots & \ddots & \vdots \\
        \bO & \bO & \cdots & \bW^{NN} \end{bmatrix}
    \begin{bmatrix} \bx^1 \\ \bx^2 \\ \vdots \\ \bx^N \end{bmatrix}
    + \begin{bmatrix} \sum_{l \neq 1} \bW^{1l} \\ \sum_{l \neq 2} \bW^{2l} \\ \vdots \\ \sum_{l \neq N} \bW^{Nl} \end{bmatrix}
    \bar{\bx},
\end{equation*}
then we get
\begin{equation*}
    \| \bW \bx - \bW_{\bal} \bx \|^2 = \sum_{k=1}^N \left\| \sum_{l \neq k} \bW^{kl} (\bx^l - \bar{\bx}) \right\|^2.
\end{equation*}
Note that, in the proof of~\cref{thm:1}, the independence condition of $\{ x_j \}$ was never used.
Hence, we can use the same argument as in the proof of \cref{thm:1} to obtain
\begin{equation}
\label{thm1:2}
    \begin{split}
        &\bbE_{\bW} \left[ \inf_{\bW_{\gp}} \bbE_{\bx} \left[ \| \bW \bx - \bW_{\gp} \bx \|^2 \right] \right] \\
        &\leq C_{n,N,k} \left( 1 - \frac{1}{N} \right)^2 \bbE \left[ \| \bW \|_{\param}^2 \right] \bbE \left[ \| \bx - \bar{\bx }\|^2 \right] \\
        &= C_{n,N,k} \left( 1 - \frac{1}{N} \right)^2 \bbE \left[ \| \bW \|_{\param}^2 \right] \left( \bbE \left[ \| \bx \|^2 \right] - N \bbE \left[ \| \bar{\bx} \|^2 \right] \right),
    \end{split}
\end{equation}
where $C_{n,N,k}$ was given in~\eqref{C}.
Meanwhile, since $\{ \bx^k \}$ are i.i.d., we have
\begin{equation}
\label{thm2:2}
    \bbE \left[ \| \bar{\bx} \|^2 \right] = \frac{1}{N^2} \bbE \left[ \| \bx \|^2 \right] + \frac{N-1}{N} \left\| \bbE \left[ \bx^1 \right] \right\|^2
    \geq \frac{1}{N^2} \bbE \left[ \| \bx \|^2 \right].
\end{equation}
Combining~\eqref{thm1:2} and~\eqref{thm2:2} yields the desired result.
\end{proof}

\begin{remark}
By \cref{ass:data}, the approximability of the group convolution has a bound of $K(1-1/N)^{2}$.
Adding \cref{ass:input} to this provides $K(1-1/N)/n$ bound.
Since $n \geq N$, the latter is a sharper bound.
On the other hand, in both cases, the approximability of BGC has sharper estimates than the group convolution by $(1-1/N)$ factor.
\end{remark}

\subsection{Computational cost}
\label{Sec:cost}
We now consider the computational cost for a single convolutional layer with the proposed BGC.
Let $D$ be the size of each input channel, i.e., $x_{j} \in \mathbb{R}^{D}$.
The number of scalar arithmetic operations required in a single operation of BGC is given by
\begin{equation}
    \label{cost}
    2N m_N n_N KD + n D = \frac{2KDmn}{N} + Dn.
\end{equation}
In the right-hand side of~\eqref{cost}, the first term corresponds to $2N$ blocks of convolutional layers in~\eqref{eqn:balanced}, and the second term corresponds to the computation of $\bar{\bx}$.
Noting that the standard convolution and the group convolution require $KDmn$ and $KDmn/N$ scalar arithmetic operations, respectively, we conclude that the computational cost of BGC is approximately $2/N$ of that of the standard convolution, but twice as much as that of the group convolution.

\section{Comparison with existing works}
\label{Sec:Comparison}
\begin{table}
\caption{Overview of various variants of group convolution~(GC),
where $n$ and $m$ denote the numbers of input and output channels, respectively, $N$ denotes the number of groups, and $K$ represents the number of parameters of the single-channel convolutional layer.}
\label{tab:existing}
\resizebox{\textwidth}{!}{
\begin{tabular}{lcccccc}
\toprule
Method & \begin{tabular}[c]{@{}c@{}}Intergroup\\ commun.\end{tabular} & \begin{tabular}[c]{@{}c@{}}\# of channels\\ per group\end{tabular} & \begin{tabular}[c]{@{}c@{}}Grouping\end{tabular} &  \# of parameters & \begin{tabular}[c]{@{}c@{}}Comput.\\cost\end{tabular} \\
\midrule \midrule
GC~\citep{krizhevsky2012imagenet} & No & $n/N$ & deterministic & $Kmn/N$ & $O(1/N)$ \\
Shuffle~\citep{zhang2018shufflenet} & Yes & $n/N$ & deterministic & $Kmn/N$ & $O(1/N)$ \\
Learnable GC~\citep{Huang_2018_CVPR} & Yes & $n / N$ & adaptive &  $2Kmn$ & greater \\
Fully learnable GC~\citep{Wang_2019_CVPR} & Yes & varies & adaptive &  $Kmn+Nn+Nm$ & greater \\
Two-level GC~\citep{lee2022two} & Yes & $n/N$ & deterministic & $Kmn/N+Kn+Nm$ & greater \\
\midrule
BGC & Yes& $n/N$ & deterministic & $2Kmn/N$ & $O(1/N)$ \\
\bottomrule                          
\end{tabular}
}
\end{table}

In this section, we compare BGC with other variants of group convolution.
Various methods have been proposed to improve the performance of group convolution by enabling intergroup communication: Shuffle~\citep{zhang2018shufflenet}, learnable group convolution~(LGC)~\citep{Huang_2018_CVPR}, fully learnable group convolution~(FLGC)~\citep{Wang_2019_CVPR}, and two-level group convolution~(TLGC)~\citep{lee2022two}.
In BGC, input and output channels are evenly partitioned into groups to ensure that all groups have the same computational cost, preventing computational bottlenecks.
This feature also applies to GC, Shuffle, LGC, and TLGC, but deterministically to BGC, i.e., it uses a fixed partition that is independent of convolution and input.
Hence, the computational cost for partitioning into blocks is less expensive than those of LGC and FLGC, which have different partitions for each convolution.

The number of parameters in BGC decreases at a rate of $O(1/N)$ as the number of groups $N$ increases, which is consistent with the rates of GC and Shuffle.
Consequently, the computational cost of BGC can be scaled by $1/N$, similar to GC and Shuffle.
This is contrast to LGC, FLGC, and TLGC, which have additional parameters that are not $O(1/N)$, making their computational costs greater than $O(1/N)$.
Note that the total computational cost of LGC, FLGC, and TLGC is $DKmn/N + (N-1)*(Kmn/N + mn/N +1)$, $DKmn/N +(n+m)N^{2} + Kmn/N$, and $DKmn/N + DKn + DNm$, respectively.

As discussed in \cref{thm:2}, the upper bound on the approximability of BGC is $K(1-1/N)^3$, which is better than GC in \cref{thm:1} by a factor of $(1-1/N)$.
To the best of our knowledge, this is the first rigorous analysis of the performance of a group convolution variant.
On the other hand, there is still no approximation theory for Shuffle, LGC, FLGC, and TLGC.

\cref{tab:existing} provides an overview of the comparison between BGC and the other variants of group convolution discussed above.
As shown in \cref{tab:existing}, Shuffle~\citep{zhang2018shufflenet} shares all the advantages of BGC except for the theoretical guarantee of the approximability.
However, BGC has another advantage over Shuffle.
While Shuffle relies on layer propagation to have intergroup communication and provide good accuracy, BGC incorporates intergroup communication in each layer, increasing accuracy even when the network is not deep enough.
This is particularly advantageous when BGC is applied to CNNs with wide but not deep layers, such as WideResNet~\citep{zagoruyko2016wide}. 
This assertion will be further supported by the numerical results in the next section.

\section{Numerical experiments}
\label{Sec:Num}
In this section, we present numerical results that verify our theoretical results and demonstrate the performance of the proposed BGC.
All programs were implemented in Python with PyTorch~\citep{paszke2019pytorch} and all computations were performed on a cluster equipped with Intel Xeon Gold 6240R
(2.4GHz, 24C) CPUs, NVIDIA RTX 3090 GPUs, and the operating system CentOS 7.

\subsection{Verification of the approximability estimates}
\label{Sub:Ver}
\begin{figure}
  \centering
  \begin{subfigure}{0.23\textwidth}
      \includegraphics[width=\textwidth]{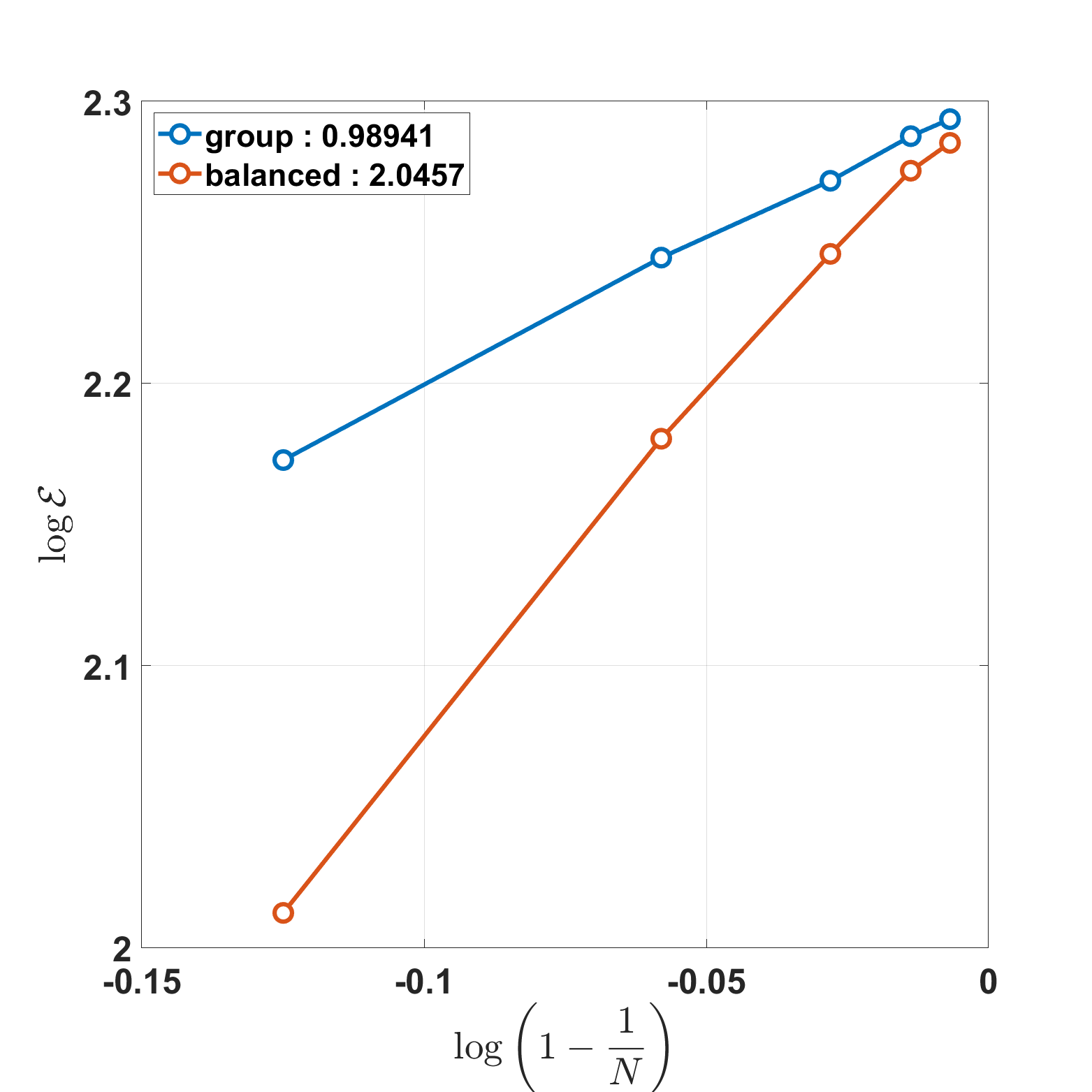}
      \caption{$S = 10^2$, $\mathcal{N}(0,1)$}
      \label{fig2:1}
  \end{subfigure}
  \hfill
  \begin{subfigure}{0.23\textwidth}
      \includegraphics[width=\textwidth]{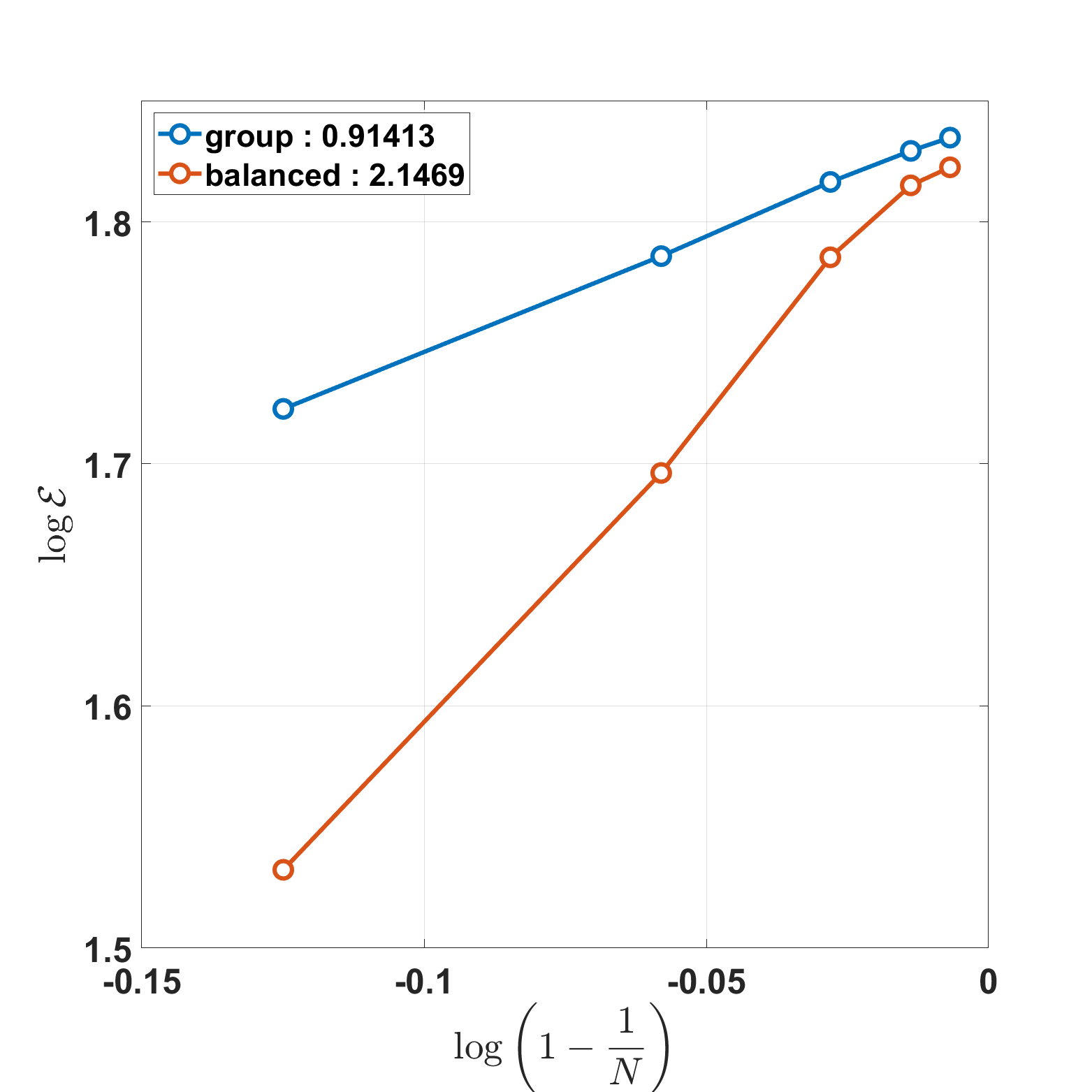}
      \caption{$S = 10^3$, $\mathcal{N}(0,1)$}
      \label{fig2:2}
  \end{subfigure}
  \hfill
  \begin{subfigure}{0.23\textwidth}
      \includegraphics[width=\textwidth]{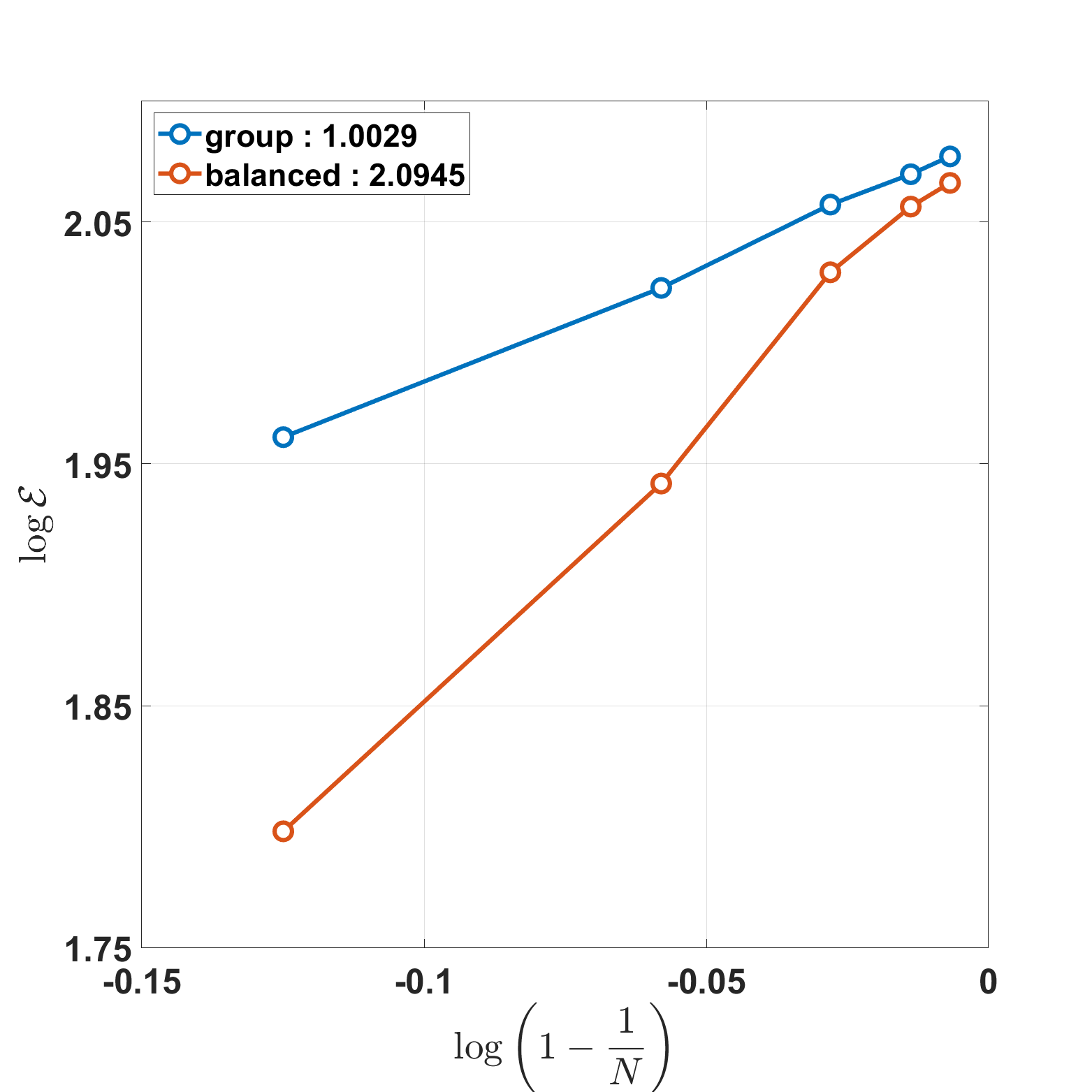}
      \caption{$S = 10^2$, $\mathcal{U}(-1,1)$}
      \label{fig2:4}
  \end{subfigure}
  \hfill
  \begin{subfigure}{0.23\textwidth}
    \includegraphics[width=\textwidth]{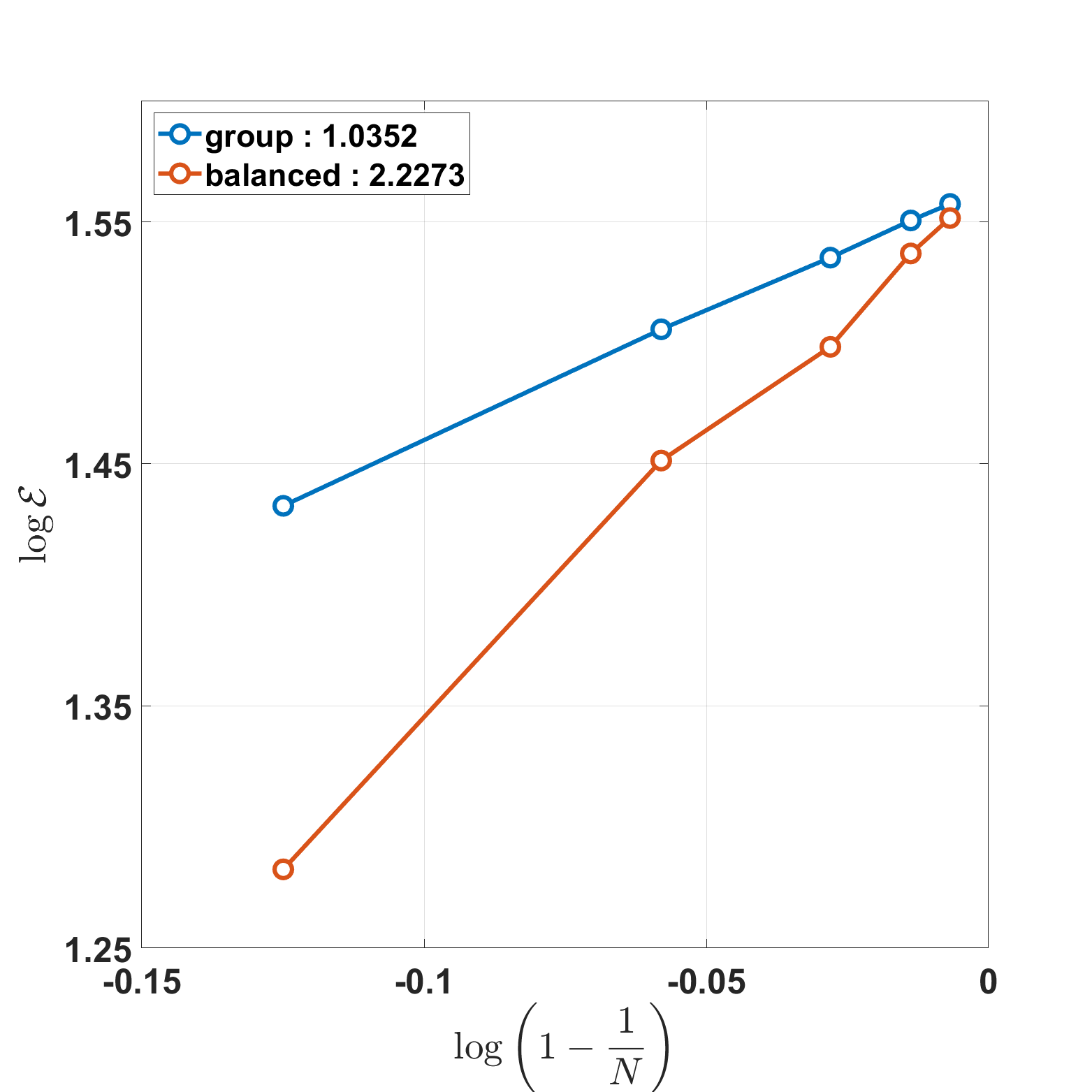}
    \caption{$S = 10^3$, $\mathcal{U}(-1,1)$}
    \label{fig2:5}
  \end{subfigure}          
  \caption{Graphs depicting $\log \E$ with respect to $\log (1 - 1/N)$ for GC and the proposed BGC, where $\E$ is the error measure given in~\eqref{eqn:lsq} and $N$ is the number of groups. We draw data points from either the normal distribution $\mathcal{N} (0, 1)$ or the uniform distribution $\mathcal{U} (-1, 1)$.
  The number in the legend indicates the average slope. Note that $N=2^{k}$, $k=2, \cdots, 6$.}
  \label{fig:scale}
  \end{figure}

\begin{figure}
  \centering
  \begin{subfigure}{0.23\textwidth}
      \includegraphics[width=\textwidth]{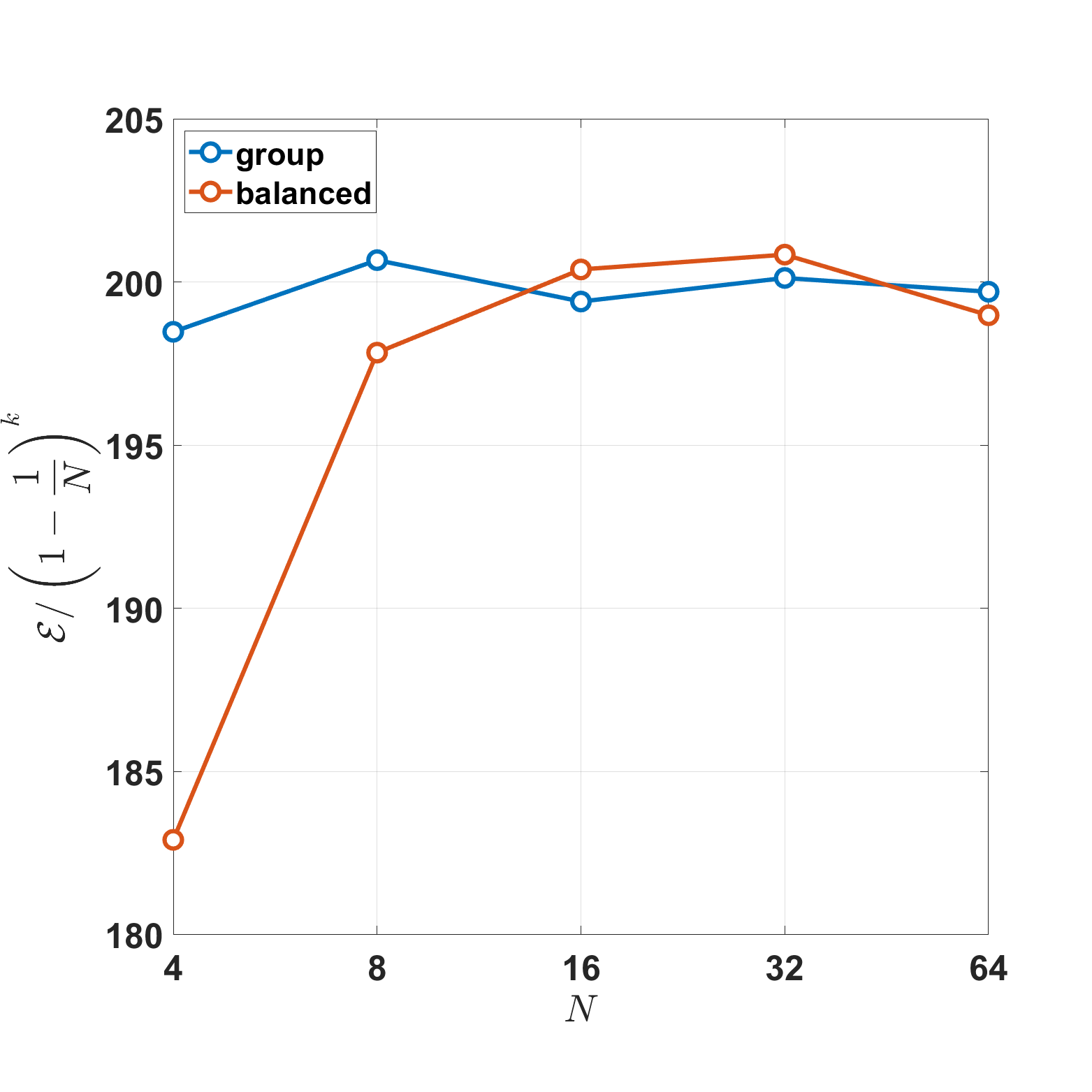}
      \caption{$S = 10^2$, $\mathcal{N}(0,1)$}
      \label{fig1:1}
  \end{subfigure}
  \hfill
  \begin{subfigure}{0.23\textwidth}
      \includegraphics[width=\textwidth]{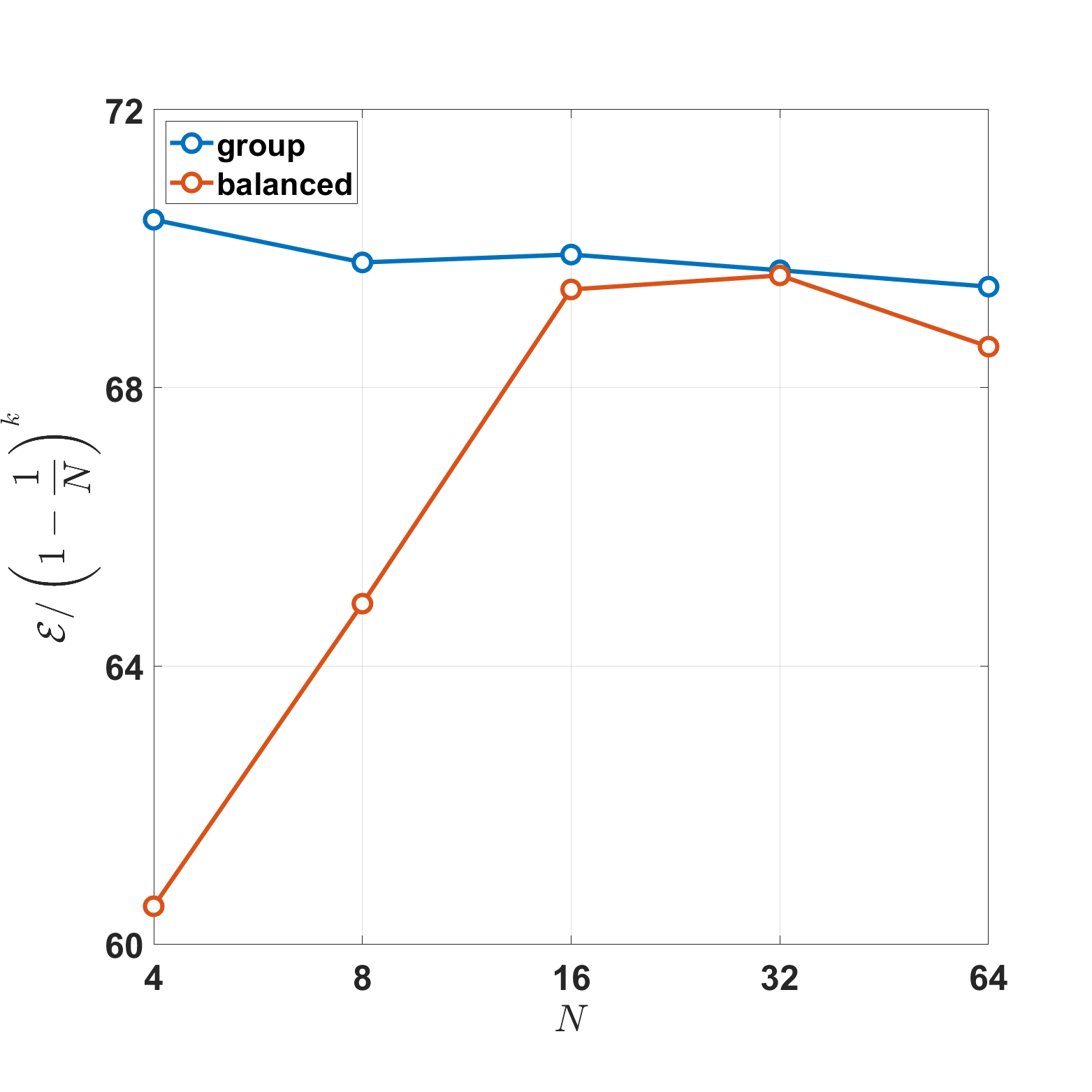}
      \caption{$S = 10^3$, $\mathcal{N}(0,1)$}
      \label{fig1:2}
  \end{subfigure}
  \hfill
  \begin{subfigure}{0.23\textwidth}
      \includegraphics[width=\textwidth]{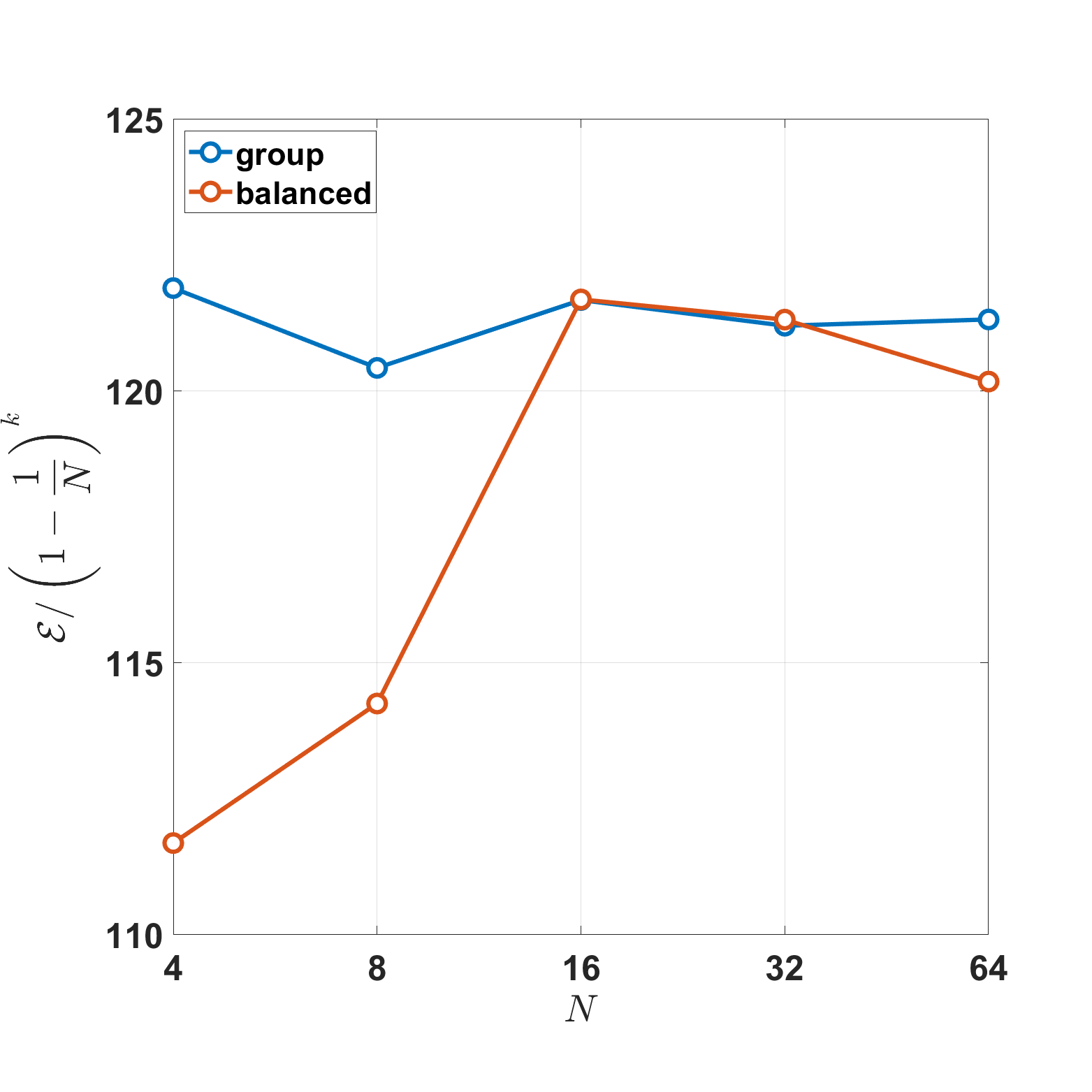}
      \caption{$S = 10^2$, $\mathcal{U}(-1,1)$}
      \label{fig1:4}
  \end{subfigure}
  \hfill
  \begin{subfigure}{0.23\textwidth}
    \includegraphics[width=\textwidth]{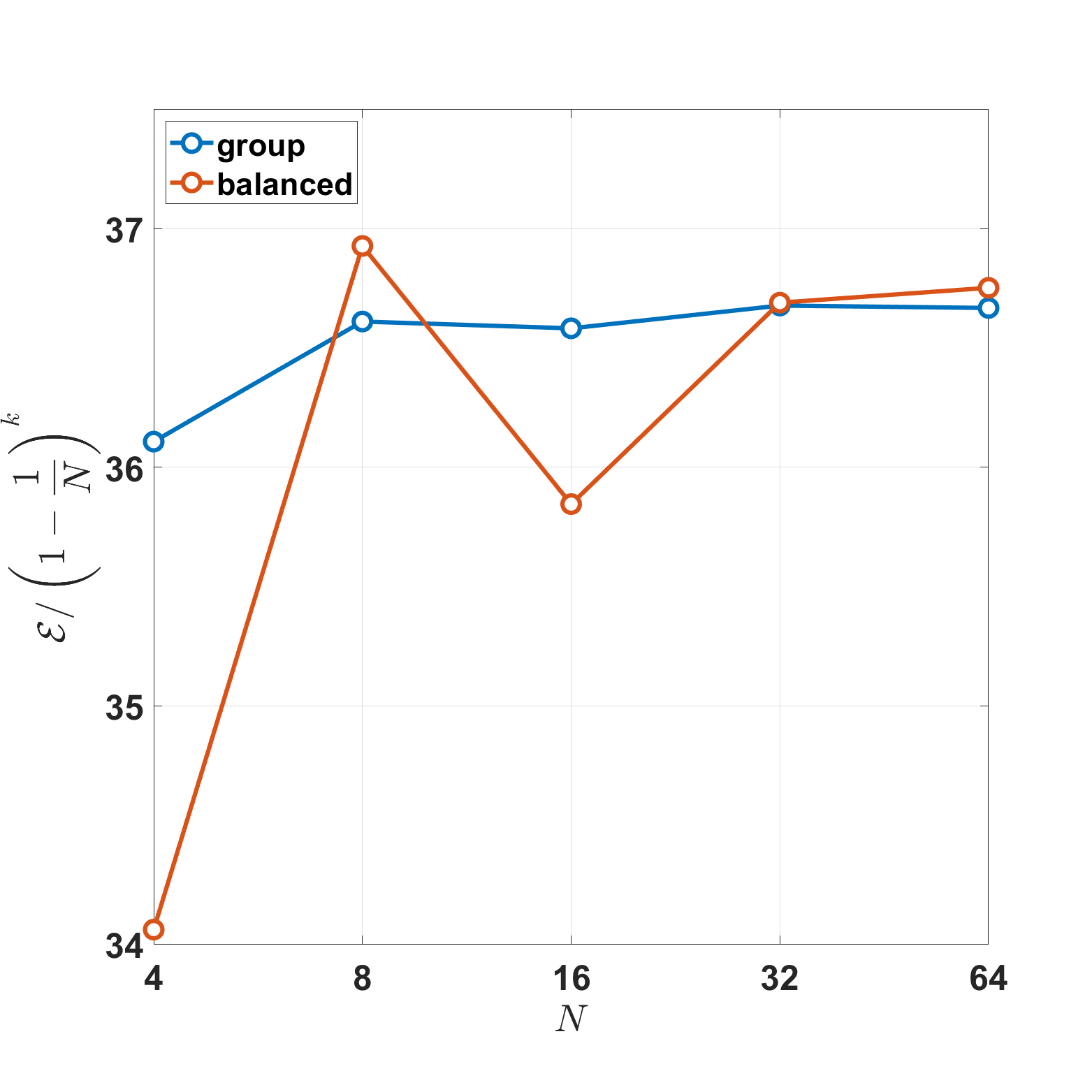}
    \caption{$S = 10^3$, $\mathcal{U}(-1,1)$}
    \label{fig1:5}
  \end{subfigure}          
  \caption{Graphs depicting $\mathrm{Rel.}\E / (1 - 1/N)^{p}$ with respect to the number of groups $N$ for GC~$(p=2)$ and the proposed BGC~$(p=3)$, where $\E$ is the error measure given in~\eqref{eqn:lsq}. We draw data points from either the normal distribution $\mathcal{N} (0, 1)$ or the uniform distribution $\mathcal{U} (-1, 1)$.}
  \label{fig:bound}
  \end{figure}

We verify the approximability estimates of GC and BGC, showing that the estimate on approximability of BGC is better than that of GC, shown in~\cref{thm:1,thm:2}.
We consider a set of one-dimensional standard convolutional layers $\{ \bW^{(t)} \}_{t=1}^{S}$ with $m = n = 256$ and $K = 3$, which is generated by He initialization~\citep{he2015delving}.
We also sample $S$ random data points $\{ \bx^{(s)} \}_{s=1}^{S}$ from either the normal distribution $\mathcal{N} (0, 1)$ or the uniform distribution $\mathcal{U} (-1, 1)$.
To measure the approximability of GC~$(\mathrm{m}=\gp)$ and BGC~$(\mathrm{m}=\bal)$ with respect to $\bW$, we use the following measure:
\begin{equation}
  \label{eqn:lsq}
  \E = \frac{1}{S} \sum_{t=1}^{S} \left( \min_{\bW_{\mathrm{m}}} \frac{1}{S} \sum_{s=1}^{S} \| \bW^{(t)} \bx^{(s)} - \bW_{\mathrm{m}} \bx^{(s)} \|^{2} \right),
\end{equation}
which can be evaluated by conventional least squares solvers~\citep{golub2013matrix}.

The graphs in \cref{fig:scale} depict $\log \E$ with respect to $\log (1 - 1/N)$ at various settings for GC and BGC.
One can readily observe linear relations between $\log \E$ and $\log (1 - 1/N)$ for both GC and BGC.
That is, we have an empirical formula
\begin{equation}
    \label{eqn:empirical}
    \E \approx C \left(1 - \frac{1}{N} \right)^{\gamma} \frac{1}{S}\sum_{t=1}^{S} \| \bW^{(t)} \|^{2} \frac{1}{S}\sum_{s=1}^{S} \| \bx^{(s)} \|^{2}
\end{equation}
for some positive constants $C$ and $\gamma$ independent of $N$.
From the graph, we see that $\gamma \approx 1$ for GC and $\gamma \approx 2$ for BGC, which confirm our theoretical results in~\cref{thm:1,thm:2}.
Next, we define a relative approximability
\begin{equation}
    \mathrm{Rel.}\E = \frac{\E}{\frac{1}{S}\sum_{t=1}^{S} \| \bW^{(t)} \|^{2} \frac{1}{S}\sum_{s=1}^{S} \| \bx^{(s)} \|^{2}}.
\end{equation}
In \cref{fig:bound}, we plot $\mathrm{Rel.}\E/(1-1/N)^{p}$ with respect to $N$ under various settings for GC and BGC, where $p=1$ for GC and $p=2$ for BGC.
We observe that all values of $\mathrm{Rel.}\E/(1-1/N)^{p}$ are bounded by $3.83 \times 10^{-3}$ independent of $N$, which implies the constant $C$ is less than $K/n \approx 11.72 \times 10^{-3}$ in~\eqref{eqn:empirical}.
This observation also agrees with the theoretical results in~\cref{thm:1,thm:2}.

\subsection{Embedding in recent CNNs}
In~\cref{Sub:Ver}, we verified that our approximability estimates are consistent with experimental results using synthetic data.
Now, we examine the classification performance of BGC applied to various recent CNNs, including ResNet~\citep{he2016deep}, WideResNet~\citep{zagoruyko2016wide}, ResNeXt~\citep{xie2017aggregated}, MobileNetV2~\citep{sandler2018mobilenetv2}, and EfficientNet~\citep{tan2019efficientnet}.
The network structures used in our experiments are described below:
\begin{itemize}
  \item \textbf{ResNet.} In our experiments, we used the ResNet-50, which was used for ImageNet classifications.
  It uses a bottleneck structure consisting of a $1 \times 1$ standard convolution, a group convolution with a $3 \times 3$ kernel, another $1 \times 1$ standard convolution, and a skip connection.
  \item \textbf{WideResNet.} In our experiments, we used two different structures of WideResNet: WideResNet-28-10 and WideResNet-34-2, which were used for CIFAR-10 and ImageNet classifications, respectively.
  It uses a residual unit consisting of two $3 \times 3$ convolutions and a skip connection.
  \item \textbf{ResNeXt.} It uses the same structure as the ResNet except using $3 \times 3$ group convolution instead of $3 \times 3$ convolution. In our experiments, we used ResNeXt-29 $(8 \times 64\text{d})$ for CIFAR-10 classification. We applied the group convolutions to two $1 \times 1$ standard convolutions.
  \item \textbf{MobileNetV2.} The basic structure of MobileNetV2~\citep{sandler2018mobilenetv2} is an inverted residual block, similar to the bottleneck in ResNeXt. However, MobileNetV2 uses a depthwise convolution~\citep{chollet2017xception} instead of ResNeXt's $3 \times 3$ group convolution. In our experiments, we applied group convolution to two $1 \times 1$ standard convolutions.
  \item \textbf{EfficientNet.} It is based on MnasNet~\citep{tan2019mnasnet}, a modification of MobileNetV2. Using an automated machine learning technique~\citep{zophneural2016}, EfficientNet proposes several model structures with appropriate numbers of channels and depth. It also has several variations, from b0 to b7 models, depending on the number of parameters. In our experiments, we used the most basic model b0.
\end{itemize}

\begin{table}
  \centering
  \caption{Details of CIFAR-10 and ImageNet datasets.}
  \label{tab:data}
  \begin{tabular}{ccccc}
  \toprule
  Dataset  & Image size & Classes & \# of training / validation samples \\ \midrule
  CIFAR-10 & $32 \times 32$ & 10 & $50{,}000$ / $10{,}000$ \\
  ImageNet & $224 \times 224$ & $1{,}000$ & $1{,}280{,}000$ / $50{,}000$ \\ \bottomrule
  \end{tabular}
\end{table}

We evaluate the performance of various variants of group convolution on the datasets CIFAR-10~\citep{krizhevsky2009learning} and ImageNet ILSVRC 2012 \citep{deng2009imagenet}.
Details on these datasets are in \cref{tab:data}.
In addition, for the CIFAR-10 dataset, a data augmentation technique in~\citep{lee2015deeply} was adopted for training; four pixels are padded on each side of images and $32 \times 32$ random crops are sampled from the padded images and their horizontal flips.
For the ImageNet dataset, input images of size $224 \times 224$ are randomly cropped from a resized image using the scale and aspect ratio augmentation of~\citep{szegedy2015going}, which was implemented by~\citep{gross2016training}.

\subsubsection{Ablation study through transfer learning}
First, we will apply GC and BGC to a pre-trained network on real data to see how well BGC works compared to GC.
We select a pre-trained ResNet-50 trained on the ImageNet dataset.
Note that the pre-trained parameters can be found in the PyTorch library~\citep{paszke2019pytorch}.
By transferring the parameters of the pre-trained standard convolution in ResNet-50, we obtained the parameters of~\eqref{eqn:group} and~\eqref{eqn:balanced}, which are referred to as ResNet-50-GC and ResNet-50-BGC, respectively.
We then further trained the ResNet-50-GC and ResNet-50-BGC for $30$ epochs using SGD optimizer with batch size $128$, learning rate $0.01$, Nesterov momentum $0.9$, and weight decay $0.0001$.

The classification performance of ResNet-50-GC and ResNet-50-BGC is given in~\cref{tab:transfer}.
Compared to GC, the classification performance of BGC improves by as little as 4\% and as much as 15\%.
Therefore, even for real data, it can be confirmed that BGC definitely complements the performance of the neural network degraded by GC.

\begin{table}
    \centering
    \caption{Classification errors~(\%) of ResNet-50-BC and ResNet-50-BGC on ImageNet dataset after transfer learning. Each network is transfer learned from the pre-trained ResNet-50 provided in PyTorch. Note that the case of $N=1$ is the classification error of ResNet-50.}
    \label{tab:transfer}
    \begin{tabular}{c|c|c}
    \toprule
    $N$  & ResNet-50-GC    & ResNet-50-BGC   \\ \midrule
    1  & \multicolumn{2}{c}{23.87} \\ \midrule
    2  & 26.54 & 22.45 \\
    4  & 36.70 & 25.51 \\
    8  & 43.72 & 31.13\\
    16 & 50.09 & 35.45 \\ \bottomrule
    \end{tabular}
\end{table}

\subsubsection{Computational efficiency}
To verify the computational efficiency of our BGC, we conducted experiments to increase the number of groups $N$ of $3 \times 3$ convolution mapping $1024$ channels to $1024$ channels with input $\bx \in \R^{128 \times 1024 \times 7 \times 7}$.
Note that this convolution is used for ResNet-50.
\cref{tab:perf} shows the total memory usage, computation time, and classification errors for forward and backward propagations of convolutions equipped with GC and BGC.
Note that the number of groups $N$ varies from $2$ to $16$.
The results are computed on a single GPU.
First, looking at the total memory usage, BGC uses more memory than GC, but the gap narrows as $N$ increases.
The additional memory usage of BGC occurs when computing the mean $\bar{\bx}$ and the convolution $\bar{\bW}$ defined in~\eqref{eqn:balanced}.
On the other hand, looking at computation time, BGC is slower than GC, but faster than the standard convolution when $N>2$.
As can be seen in~\eqref{cost}, compared to GC, BGC requires more computation time because it performs convolution twice, but as $N$ increases, the time decreases.
Moreover, we can see that the cost of calculating the mean $\bar{\bx}$ is small enough that it has little effect on the total computation time.
Eventually, these results suggest that while BGC increases memory consumption and computation time compared to GC, it can improve performance with only a small increase in the computational cost when dealing with large channels in group convolution.

\begin{table}
\centering
\caption{Total memory usage, computation time, and classification errors~(\%) for forward and backward propagations of $3 \times 3$ convolution from $1024$ channels to $1024$ channels equipped with GC and BGC and the input $\bx \in \R^{128 \times 1024 \times 7\times 7}$.
Note that the case of $N=1$ is computed with the standard convolution.}
\label{tab:perf}
\resizebox{\textwidth}{!}{
\begin{tabular}{c|cc|cc|cc}
\toprule
\multirow{2}{*}{N} & \multicolumn{2}{c|}{Total memory usage~(Mb)} & \multicolumn{2}{c|}{Computation time~(ms)} & \multicolumn{2}{c}{Classification error~(\%)} \\ \cmidrule(l){2-7} 
 & GC & BGC & GC & BGC & GC & BGC \\ \midrule\midrule
1 & \multicolumn{2}{c|}{60.50} & \multicolumn{2}{c|}{7.979} & \multicolumn{2}{c}{23.87} \\ \midrule
2 & 42.50 & 75.75 & 9.056 & 10.467 & 26.54 & 22.45 \\
4 & 33.50 & 48.62 & 4.449 & 4.790 & 36.70 & 25.51 \\
8 & 29.00 & 37.06 & 2.120 & 2.648 & 43.72 & 31.13 \\
16 & 26.75 & 31.53 & 2.265 & 2.557 & 50.09 & 35.45 \\ \bottomrule
\end{tabular}
}
\end{table}

\subsubsection{Comparison with existing approaches}

\begin{table}
  \centering
  \caption{Classification errors~(\%) of EfficientNet-b0, WideResNet-28-10, and ResNeXt-29 $(8 \times 64\text{d})$ on the CIFAR-10 dataset equipped with GC, Shuffle, fully learnable group convolution~(FLGC), two-level group convolution~(TLGC), and BGC.
  The case of $N=1$ corresponds to the standard convolution~(SC).
  Note that the EfficientNet-b0 could not increase $N$ up to $16$ due to its structure, so is was implemented only up to $8$.}
  \label{tab:comp}
  \begin{tabular}{c|c|ccc}
    \toprule
    $N$ & Method & EfficientNet-b0 & WideResNet-28-10 & ResNeXt-29 ($8 \times 64\text{d}$) \\ \midrule\midrule
    1 & SC & 8.42 & 3.88 & 5.32 \\ \midrule
    \multirow{5}{*}{2}  & GC & 10.02 & 4.33 & 5.92 \\
                        & Shuffle & 8.72 & 4.00 & \textbf{4.52} \\
                        & FLGC & 8.63 & 4.62 & 4.68 \\
                        & TLGC & 8.93 & 4.12 & 5.16 \\
                        & BGC & \textbf{7.16} & \textbf{3.75} & 5.84 \\ \midrule
    \multirow{5}{*}{4}  & GC & 11.25 & 5.44 & 6.76 \\
                        & Shuffle & 9.41 & 4.21 & \textbf{4.17} \\
                        & FLGC & 9.70 & 6.16 & 5.16 \\
                        & TLGC & 9.42 & 4.34 & 6.92 \\
                        & BGC & \textbf{7.50} & \textbf{4.02} & 5.84 \\ \bottomrule
    \end{tabular}
  \begin{tabular}{c|c|ccc}
    \toprule
    $N$ & Method & EfficientNet-b0 & WideResNet-28-10 & ResNeXt-29 ($8 \times 64\text{d}$) \\ \midrule\midrule
    1 & SC & 8.42 & 3.88 & 5.32 \\ \midrule
    \multirow{5}{*}{8}  & GC & 13.14 & 6.43 & 6.52 \\
                        & Shuffle & 10.65 & 4.73 & \textbf{4.81} \\
                        & FLGC & 10.49 & 9.81 & 5.25 \\
                        & TLGC & 10.68 & 4.79 & 6.56 \\
                        & BGC & \textbf{7.80} & \textbf{4.22} & 6.00 \\ \midrule
    \multirow{5}{*}{16} & GC & \multirow{5}{*}{-} & 8.47 & 6.92 \\
                        & Shuffle & & 5.28 & \textbf{4.89} \\
                        & FLGC & & 11.26 & 6.11 \\
                        & TLGC & & 4.95 & 6.08 \\
                        & BGC & & \textbf{4.76} & 5.60 \\ \bottomrule
    \end{tabular}
\end{table}

\begin{table}
  \centering
  \caption{Classification errors~(\%) of EfficientNet-b0, WideResNet-34-2, and MobileNetV2 on the ImageNet dataset equipped with GC, Shuffle, fully learnable group convolution~(FLGC), two-level group convolution~(TLGC), and BGC.
  The case of $N=1$ corresponds to the standard convolution~(SC) equipped to CNNs.}
  \label{tab:comp2}
  \begin{tabular}{c|c|ccc}
  \toprule
  $N$ & Method & EfficientNet-b0 & WideResNet-34-2 & MobileNetV2 \\ \midrule \midrule
  1 & SC & 30.68 & 24.99 & 34.13 \\ \midrule
  \multirow{5}{*}{2} & GC & 35.18 & 26.92 & 39.50 \\
                     & Shuffle & 33.91 & 25.12 & 36.67 \\
                     & FLGC & 35.53 & 30.53 & 37.98 \\
                     & TLGC & 33.50 & 25.36 & \textbf{33.56} \\
                     & BGC & \textbf{31.82} & \textbf{24.19} & 36.10 \\ \midrule
  \multirow{5}{*}{4} & GC & 41.64 & 31.96 & 46.15 \\
                     & Shuffle & 38.02 & 26.84 & 38.14 \\
                     & FLGC & 40.71 & 38.30 & 44.37 \\
                     & TLGC & 36.38 & 27.17 & \textbf{36.88} \\
                     & BGC & \textbf{33.78} & \textbf{25.36} & 37.82 \\ \midrule
  \multirow{5}{*}{8} & GC & 48.42 & 36.70 & 51.97 \\
                     & Shuffle & 42.74 & 29.03 & 42.75 \\
                     & FLGC & 45.38 & 45.54 & 49.28 \\
                     & TLGC & 38.66 & 28.43 & \textbf{39.16} \\
                     & BGC & \textbf{38.13} & \textbf{27.16} & 42.65 \\ \bottomrule
  \end{tabular}
\end{table}

As benchmark group convolution variants, we choose GC~\citep{krizhevsky2012imagenet}, Shuffle~\citep{zhang2018shufflenet}, FLGC~\citep{Wang_2019_CVPR}, and TLGC~\citep{lee2022two}, which were discussed in \cref{Sec:Comparison}.
All neural networks for the CIFAR-10 dataset in this section were trained using stochastic gradient descent with batch size $128$, weight decay $0.0005$, Nesterov momentum $0.9$,
total epoch $200$, and weights initialized as in~\citep{he2015delving}.
The initial learning rate was set to $0.1$ and was reduced to its one tenth in the 60th, 120th, and 160th epochs.
For ImageNet, the hyperparameter settings are the same as the CIFAR case, except for the weight decay $0.0001$, total epoch $90$, and the learning rate reduced by a factor of 10 in the 30th and 60th epochs.

The classification errors of BGC, along with other benchmarks, applied to various CNNs on the CIFAR-10 dataset are presented in \cref{tab:comp}.
BGC shows better overall results than other methods.
In particular, the application of BGC significantly improves the classification performance of EfficientNet-b0.
In addition, the classification errors of BGC for WideResNet-28-10 are always less than 5\% when $N$ varies up to $16$.
Although Shuffle performed the best for ResNeXt-29, BGC still outperforms the classification performance of GC.
To further validate the performance of BGC, we report the classification results on the ImageNet dataset in \cref{tab:comp2}.
For this dataset, we observe that BGC outperforms other benchmark group convolution variants for CNN architectures except MobileNetV2.
Therefore, through several experiments, we conclude that BGC is an effective and theoretically guaranteed alternative to group convolution for various CNN architectures on large-scale datasets.

\section{Conclusion}
\label{Sec:Conc}
In this paper, we proposed a novel variant of group convolution called BGC.
We constructed BGC by combining the plain group convolution structure with a balancing term defined as the intergroup mean to improve intergroup communication.
We designed a new measure~\eqref{eqn:app_err} to assess the approximability of group convolution variants and proved that the approximability of group convolution is bounded by $K(1-1/N)^{2}$.
Also, we showed that the bound on approximability of proposed BGC is $K(1-1/N)^{3}$, which is an improved bound compared to the group convolution.
Moreover, under the additional assumption about the parameters of the standard convolution, we showed that the bounds for the approximability of group convolution and BGC are $K(1-1/N)/n$ and $K(1-1/N)^{2}/n$, respectively.
Numerical experiments with various CNNs such as WideResNet, MobileNetV2, ResNeXt, and EfficientNet have demonstrated the practical efficacy of BGC on large-scale neural networks and datasets.

We conclude this paper with a remark on BGC.
A major drawback of the proposed BGC is that it requires full data communication among groups.
This means that when computing the intergroup mean $\bar{\bx}$ that appears in the balancing term, we need the entire input $\bx$.
This high volume of communication can be a bottleneck in parallel computation, which limits the performance of the model in distributed memory systems.
We note that TLGC~\citep{lee2022two} has addressed this issue by minimizing the amount of communication required.
Exploring how to improve BGC by reducing communication in a similar manner to~\citep{lee2022two}, while maintaining strong performance in both theory and practice, is considered as a future work.

\section*{Acknowledgments}
This work was supported in part by the National Research Foundation~(NRF) of Korea grant funded by the Korea government~(MSIT) (No.~RS-2023-00208914),
and in part by Basic Science Research Program through NRF funded by the Ministry of Education (No.~2019R1A6A1A10073887).



\bibliographystyle{elsarticle-harv} 
\bibliography{bgc}





\end{document}